\documentclass[letterpaper, 10 pt, conference]{ieeeconf}  % Comment this line out if you need a4paper

\IEEEoverridecommandlockouts                              % This command is only needed if 
\usepackage{stfloats}
\usepackage{multirow}
\usepackage{graphics} % for pdf, bitmapped graphics files
\usepackage{epsfig} % for postscript graphics files
\usepackage{float}
\usepackage{placeins} 
\usepackage{subcaption}
\usepackage{cite}

\let\labelindent\relax
\usepackage{amsthm}
\usepackage{amsmath,amssymb,mathtools}
\usepackage{booktabs}
\usepackage[ruled,vlined,linesnumbered]{algorithm2e}
\usepackage{xspace}

\newtheorem{lemma}{Lemma}

\usepackage{xcolor}
\usepackage{enumitem} % compact itemize
\usepackage{algorithmic}
\newtheorem{definition}{Definition}
\newtheorem{theorem}{Theorem}

\usepackage[table]{xcolor}
\newcommand{\Rob}{\rho}                            % Exact formula-level robustness
\newcommand{\sRob}{\widetilde{\rho}}               % Smoothed formula-level robustness
\newcommand{\dov}{\delta_{\mathrm{ov}}}            % overlap slack
\newcommand{\din}{\delta_{\mathrm{in}}}            % enclosure slack
\usepackage{tikz}
\usetikzlibrary{shapes, arrows.meta, positioning, fit, backgrounds}

\newcommand{\Touch}{\mathrm{Touch}}
\newcommand{\LeftOf}{\mathrm{LeftOf}}
\newcommand{\RightOf}{\mathrm{RightOf}}
\newcommand{\Above}{\mathrm{Above}}
\newcommand{\Below}{\mathrm{Below}}
\newcommand{\CloseE}{\mathrm{closeTo}_\varepsilon}
\newcommand{\FarE}{\mathrm{farFrom}_\varepsilon}
\newcommand{\Ovlp}{\mathrm{ovlp}}
\newcommand{\PartOvlp}{\mathrm{partOvlp}}
\newcommand{\EnclIn}{\mathrm{enclIn}}
\newcommand{\Between}{\mathrm{Between}}

\newcommand{\ent}[1]{\texttt{#1}}       % entities like robot, cup
\newcommand{\smin}{\widetilde{\min}}
\newcommand{\smax}{\widetilde{\max}}

\definecolor{deepblue}{HTML}{003399}
\newcommand{\Behind}{\mathrm{Behind}}
\newcommand{\InFront}{\mathrm{InFrontOf}}
\usepackage[colorlinks=true]{hyperref}
\hypersetup{
    urlcolor = deepblue,
    linkcolor = deepblue,
    citecolor = deepblue
}
\allowdisplaybreaks

\title{\LARGE \bf
Differentiable SpaTiaL: Symbolic Learning and Reasoning with Geometric Temporal Logic for Manipulation Tasks
}

\author{Licheng Luo$^{1}$, Kaier Liang$^{2}$, Cristian-Ioan Vasile$^{2}$, Mingyu Cai$^{1}$%
\thanks{$^{1}$UC Riverside {\tt\small\{lichengl,mingyuc\}@ucr.edu}. $^{2}$Lehigh University {\tt\small\{kal221,cvasile\}@lehigh.edu}. Code Available: \url{https://github.com/plen1lune/DiffSpaTiaL}}%
}

\begin{document}

\maketitle
\thispagestyle{empty}
\pagestyle{empty}

\begin{abstract}
Executing complex manipulation in cluttered environments requires satisfying coupled geometric and temporal constraints. Although Spatio-Temporal Logic (SpaTiaL) offers a principled specification framework, its use in gradient-based optimization is limited by non-differentiable geometric operations. Existing differentiable temporal logics focus on the robot’s internal state and neglect interactive object–environment relations, while spatial logic approaches that capture such interactions rely on discrete geometry engines that break the computational graph and preclude exact gradient propagation.
To overcome this limitation, we propose \emph{Differentiable SpaTiaL}, a fully tensorized toolbox that constructs smooth, autograd-compatible geometric primitives directly over polygonal sets. To the best of our knowledge, this is the first end-to-end differentiable symbolic spatio-temporal logic toolbox. By analytically deriving differentiable relaxations of key spatial predicates—including signed distance, intersection, containment, and directional relations—we enable an end-to-end differentiable mapping from high-level semantic specifications to low-level geometric configurations, without invoking external discrete solvers. This fully differentiable formulation unlocks two core capabilities: (i) massively parallel trajectory optimization under rigorous spatio-temporal constraints, and (ii) direct learning of spatial logic parameters from demonstrations via backpropagation. Experimental results validate the effectiveness and scalability of the proposed framework.
\end{abstract}

%%%%%%%%%%%%%%%%%%%%%%%%%%%%%%%%%%%%%%%%%%%%%%%%%%%%%%%%%%%%%%%%%%%%%%%%%%%%%%%%
\section{INTRODUCTION}

% 1. 目标与背景 (Context & Objective)
Robotic manipulation in human-shared and cluttered environments demands behaviors that satisfy rich, tightly coupled spatio-temporal constraints. Beyond simple collision avoidance, robots must maintain safety margins around dynamic obstacles, operate within confined workspaces, and enforce structured geometric relations over time—for example, eventually placing object 
A inside region 
B while keeping it strictly to the left of object 
C, or transporting an object without intersecting restricted zones throughout execution. Such requirements are inherently relational, involving multi-object interactions, spatial topology, and temporal ordering.
Signal Temporal Logic (STL)~\cite{maler2004monitoring, fainekos2009robustness, donze2010robust}, augmented with spatial set-theoretic semantics as in Spatio-Temporal Logic (SpaTiaL)~\cite{pek2023spatial,luo2026nl2spatialgeneratinggeometricspatiotemporal,belta2017formal}, provides a rigorous and expressive formalism to encode these specifications. This logical framework enables precise reasoning over time-indexed geometric predicates and supports robustness metrics that quantify satisfaction. However, despite its expressiveness, a fundamental challenge in modern robotics is to integrate such high-level logical structure directly into gradient-based trajectory optimization and deep learning pipelines, thereby enabling end-to-end synthesis and inference of complex behaviors~\cite{ ratliff2009chomp, zucker2013chomp, kalakrishnan2011stomp, schulman2014motion, mukadam2018gpmp2}.
The key limitation lies in the incompatibility between geometric spatial reasoning and gradient-based backpropagation. Spatial predicates are typically evaluated through discrete collision detection, combinatorial topology checks, or non-smooth distance queries, all of which break differentiability. Recent advances in differentiable temporal logic, such as STLCG~\cite{leung2021backpropagationsignaltemporallogic, kapoor2025stlcg++}, have successfully reformulated temporal operators as smooth computation graphs, enabling gradient propagation through logical structure. Yet, these approaches operate primarily over system state variables and do not natively capture multi-dimensional geometric extents, object–object interactions, continuous object boundaries, or spatial topological relations. As a result, the expressive power of spatio-temporal logic remains largely disconnected from modern differentiable optimization and learning frameworks.

\begin{figure*}[t]
    \centering
    % Placeholder for Fig 1: Teaser
    \includegraphics[width=\textwidth]{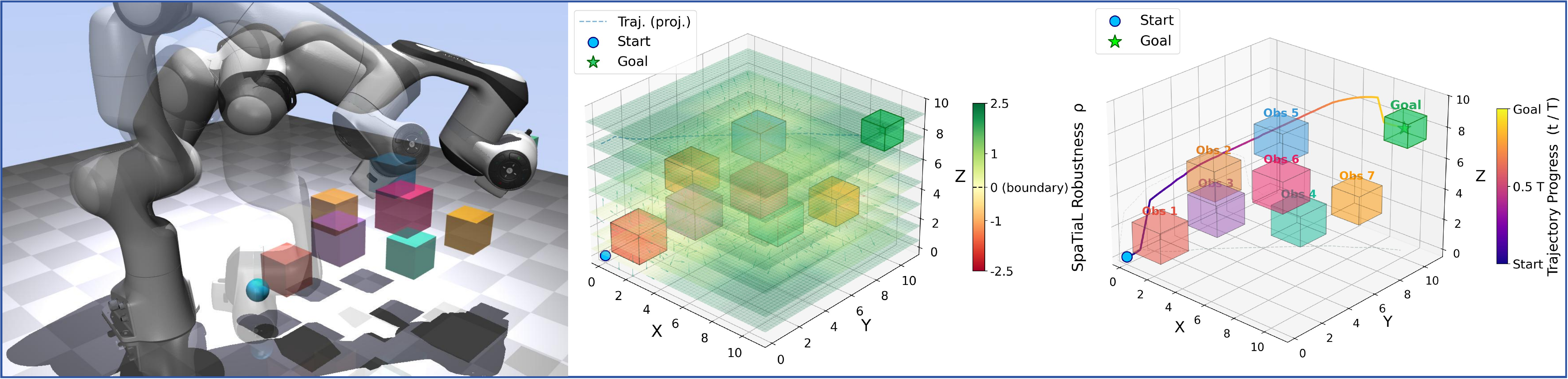} 
    \caption{Overview of \emph{Differentiable SpaTiaL}. Our framework replaces discrete geometry engines with a fully tensorized architecture, enabling end-to-end trajectory optimization under formal specifications.}
    \label{fig:teaser}
    % \vspace{-2em}
\end{figure*}

% 3. 痛点 - 几何层 (The Gap - Spatial side)
Conversely, existing spatial logic frameworks that evaluate geometric relations rely heavily on discrete collision checkers or classical geometry engines (e.g., Shapely~\cite{shapely_docs}, FCL~\cite{pan2012fcl}). These engines utilize discrete algorithmic branches and non-differentiable boolean operations. Consequently, the evaluation of spatial predicates fractures the computational graph, precluding gradients from flowing backward from the logical robustness loss to the physical robot states. To bypass this, prior works are often forced to employ grid-based spatial discretizations or heuristic gradient maps, which suffer drastically from the curse of dimensionality and remain computationally intractable for high-degree-of-freedom manipulation tasks~\cite{gilbert1988gjk,ericson2004realtime}.
% \cite{gottschalk1996obbtree,gilbert1988gjk,ericson2004realtime}.

% 4. 方案 (The Solution) - 强调 Toolbox
To overcome this limitation, we introduce \emph{Differentiable SpaTiaL}. We mathematically formulate differentiable approximations of the Separating Axis Theorem (SAT)~\cite{ericson2004realtime} and soft boundary representations for convex polygonal sets, in the spirit of differentiable geometric computation~\cite{belbute2018differentiable,degrave2019differentiable,hu2020difftaichi,montaut2023diffcol,lelidec2024simgradients}. By natively defining core spatial predicates—such as distance, signed penetration, overlap—as fully differentiable operations, \emph{Differentiable SpaTiaL} establishes an unbroken computational graph. This circumvents the discrete bottlenecks of traditional collision checkers, allowing the spatial robustness signal to be differentiated analytically with respect to the continuous poses and geometries of the physical objects, as illustrated in Fig.~\ref{fig:teaser}.

To highlight our core positioning, Table~\ref{tab:toolbox_comparison} contrasts our framework with existing approaches. While differentiable temporal-logic toolboxes such as STLCG~\cite{leung2021backpropagationsignaltemporallogic} excel at temporal reasoning and support gradient-based optimization over temporal specifications on real-valued signals, they do not natively provide object-centric geometric predicates. Differentiable geometry engines such as DiffCol~\cite{montaut2023diffcol} focus on geometric collision queries and derivatives rather than compositional temporal-logic semantics. To the best of our knowledge, \emph{Differentiable SpaTiaL} is the first end-to-end differentiable spatio-temporal logic toolbox designed for gradient-based continuous optimization.

\begin{table}[h]
\centering
\caption{Feature comparison of our toolbox against existing formal logic and geometric frameworks.}
\label{tab:toolbox_comparison}
\resizebox{\columnwidth}{!}{%
\begin{tabular}{l c c c c}
\toprule
\textbf{Framework / Toolbox} & \textbf{Temporal} & \textbf{Spatial} & \textbf{Differentiable} & \textbf{Batch-native tensorized} \\
\midrule
Diff-TL (e.g., STLCG~\cite{leung2021backpropagationsignaltemporallogic, kapoor2025stlcg++}) & \checkmark & $\times$ & \checkmark & \checkmark \\
SpaTiaL~\cite{pek2023spatial} & \checkmark & \checkmark & $\times$ & $\times$ \\
Diff. Collision (e.g., DiffCol~\cite{montaut2023diffcol}) & $\times$ & \checkmark & \checkmark & $\times$ \\
\rowcolor{gray!15} \textbf{Diff-SpaTiaL (Ours)} & \textbf{\checkmark} & \textbf{\checkmark} & \textbf{\checkmark} & \textbf{\checkmark} \\
\bottomrule
\end{tabular}%
}
\vspace{-1em}
\end{table}

% 5. 贡献 (Contributions)
By unifying geometric reasoning with automatic differentiation, our framework provides a scalable mathematical foundation for spatio-temporal logic in modern learning pipelines. The main contributions of this paper are:
\begin{itemize}[leftmargin=*, noitemsep, topsep=0pt]
    \item We develop \emph{Differentiable SpaTiaL}, to the best of our knowledge the first fully tensorized toolbox for geometric spatial predicates based on smooth SAT and boundary-sampled SDF, enabling dense gradients while avoiding discrete grid-based or iterative search bottlenecks.
    \item We present a unified gradient-based framework that integrates scalable multi-object spatio-temporal constraints into a single optimization objective, enabling formal specifications to be incorporated into GPU-based optimization workflows.
    \item We introduce a principled approach for inferring and refining SpaTiaL specifications by optimizing continuous geometric parameters directly from human demonstrations, achieved through end-to-end backpropagation over spatial robustness metrics.
\end{itemize}
\section{RELATED WORK}

\subsection{Differentiable Temporal Logic and Robustness}
STL is widely used to specify operational requirements in continuous control~\cite{maler2004monitoring,fainekos2009robustness,donze2010robust,belta2017formal}. Building on this, extensive research has leveraged STL for trajectory optimization and generation, employing techniques ranging from Mixed-Integer Linear Programming (MILP)~\cite{raman2014model} and probabilistic approaches under uncertainty~\cite{sadigh2016safe}, to continuous smoothing~\cite{pant2017smooth} and Recurrent Neural Networks (RNNs)~\cite{liu2022recurrent}. A pivotal development in scaling these methods is differentiable robustness computation~\cite{leung2021backpropagationsignaltemporallogic}, which approximates non-smooth $\min/\max$ operators to allow analytic gradient propagation through temporal operators. While these frameworks successfully tensorize the temporal dimension, extending them to natively handle complex geometric sets remains an open challenge. They often assume atomic predicates are pre-defined, differentiable scalar functions of low-dimensional states, which complicates the direct evaluation of multi-dimensional physical shapes (e.g., polygons). 

\subsection{Differentiable Collision Detection and Geometry}
To integrate geometric reasoning into gradient-based pipelines, recent advancements in differentiable physics compute gradients through collision detection algorithms (e.g., GJK/EPA-style convex proximity pipelines~\cite{gilbert1988gjk,ericson2004realtime}) via randomized smoothing, differentiable simulation, or implicit differentiation~\cite{belbute2018differentiable,degrave2019differentiable,hu2020difftaichi,montaut2023diffcol,lelidec2024simgradients}. Although these methods maintain exact geometric witness points, their iterative nature introduces challenges for vectorization. This computational overhead can be a limiting factor for batch-heavy, long-horizon trajectory optimization~\cite{pan2012fcl,schulman2014motion,mukadam2018gpmp2}. Similarly, while the foundational SpaTiaL framework~\cite{pek2023spatial} successfully pioneers spatio-temporal logic for task planning, its spatial predicates traditionally rely on these discrete geometric evaluations. To overcome this bottleneck, \emph{Differentiable SpaTiaL} relaxes the need for exact discrete iterations. By formulating continuous approximations of the SAT and soft boundary representations, the proposed architecture natively supports massive GPU parallelism, facilitating the computational efficiency required for deep learning pipelines.

\subsection{Learning and Inference of Logical Specifications}
To complement opaque neural policies, a parallel line of research infers interpretable logical specifications from demonstrations~\cite{jha-fmsd19,pmlr-v155-puranic21a,vazquez2018learning,bombara2016decision,kong2014temporal,temporal_tree_2021,enumerative_2020,census_stl_2020}. These methods optimize discrete formula structures and continuous parameters using mixed-integer programming~\cite{liang2024learning}, decision-tree and enumerative approaches~\cite{bombara2016decision,temporal_tree_2021,enumerative_2020}, or neural-symbolic networks~\cite{li2023learning,tlinet}, and more recently incorporate conformal prediction to improve inference reliability under uncertainty and covariate shift~\cite{conformal_prediction_stl,conformal_stl_covshift}, typically guided by continuous robustness~\cite{fainekos2009robustness,donze2010robust,leung2021backpropagationsignaltemporallogic}. Recent efforts further integrate logical inference into robotics and trajectory monitoring pipelines~\cite{pmlr-v155-puranic21a,logic_monitor_trajectory}. However, extending these frameworks to complex physical environments remains challenging. Exact geometric evaluations can introduce substantial computational overhead in discrete optimization, while their non-differentiable nature can hinder gradient propagation in learning-based pipelines. As a result, prior methods are often applied to simplified or non-geometric state predicates. By establishing a natively differentiable geometric foundation, our framework helps address this limitation. It tensorizes spatial set reasoning, enabling continuous geometric parameters—such as target boundaries and safety margins—to be inferred via backpropagation through spatial robustness.

\section{Problem Formulation}
This section introduces the STL\cite{maler2004monitoring} and SpaTiaL\cite{pek2023spatial} preliminaries used in this work, and then states our differentiable robustness formulation problem.

\subsection{Signal Temporal Logic}
\label{subsec:background-stl}
We consider a discrete-time finite-horizon trajectory $\xi = \{s_0, s_1, \dots, s_T\}$, where each state $s_t \in \mathbb{R}^n$ represents the geometric configurations of the robotic system and all objects in the workspace at time step $t$. STL formulas $\phi$ specify temporal properties over the trajectory $\xi$. An atomic predicate is defined as $\mu := a^\top s_t \geq b$ with $a \in \mathbb{R}^n$ and $b \in \mathbb{R}$. The grammar of STL is:
\begin{equation}\label{eq:stl-syntax}
\begin{aligned}
\phi \;::=\; \mu
\ \mid\  \neg\phi
\ \mid\  \phi_1 \wedge \phi_2
\ \mid\  \phi_1 \vee \phi_2\\
\ \mid\  F_{[t+k_1, t+k_2]}\phi
\ \mid\  G_{[t+k_1, t+k_2]}\phi
\ \mid\  \phi_1 U_{[t+k_1,t+k_2]} \phi_2
\end{aligned}
\end{equation}
where $0 \le k_1 \le k_2$ are integer time bounds representing the evaluation window $[t+k_1, t+k_2]$. The operators $F, G$, and $U$ denote \textit{eventually}, \textit{always}, and \textit{until}, respectively.

Quantitative robustness $\Rob(\xi, \phi, t)$ measures the degree of satisfaction: $\Rob(\xi, \phi, t) > 0$ if $\xi$ satisfies $\phi$ at time $t$, and $\Rob(\xi, \phi, t) < 0$ otherwise. This metric enables the transformation of logical requirements into continuous optimization objectives~\cite{donze2010robust}. For a trajectory $\xi$, the exact robustness is defined recursively:
\begin{align}
\Rob(\xi, \mu, t) &= a^\top s_t - b \label{eq:stl-robust}\\
\Rob(\xi, \neg\phi, t) &= -\Rob(\xi, \phi, t), \nonumber\\
\Rob(\xi, \phi_1 \wedge \phi_2, t) &= \min\bigl(\Rob(\xi, \phi_1, t), \Rob(\xi, \phi_2, t)\bigr), \nonumber\\
\Rob(\xi, G_{[k_1,k_2]}\phi, t) &= \min_{t' \in [t+k_1, t+k_2]} \Rob(\xi, \phi, t'), \nonumber\\
\Rob(\xi, F_{[k_1,k_2]}\phi, t) &= \max_{t' \in [t+k_1, t+k_2]} \Rob(\xi, \phi, t'). \nonumber
\end{align}
By convention, the system satisfies the specification if the initial robustness is positive, i.e., $\Rob(\xi, \phi, 0) > 0$.

\subsection{SpaTiaL: Geometry-based Spatial Predicates}
\label{subsec:background-spatial}
We adopt the SpaTiaL fragment~\cite{pek2023spatial} over compact convex sets $\mathcal{P}_i \subset \mathbb{R}^n$ ($n \in \{2,3\}$), whose vertices and headings define the scene state $s_t$. Its spatial atoms are derived from signed distances and axis-aligned projections, yielding quantitative semantics compatible with STL monitoring.

\begin{definition}[SpaTiaL Atomic Predicates~\cite{pek2023spatial}]
For distinct objects $i,j$ and constants $\varepsilon, \varepsilon_c,\varepsilon_f,\dov,\din,\kappa>0$, let $\mathrm{dist}(\mathcal{P}_i,\mathcal{P}_j) = \min_{u \in \mathcal{P}_i, v \in \mathcal{P}_j} \|u-v\|$ be the minimum Euclidean distance. Let $\mathrm{pen}(\mathcal{P}_i,\mathcal{P}_j)$ denote the penetration depth. We define the signed clearance as $\mathrm{clr}_{ij} = \mathrm{dist}(\mathcal{P}_i,\mathcal{P}_j)$ if $\mathcal{P}_i \cap \mathcal{P}_j = \emptyset$, and $\mathrm{clr}_{ij} = -\mathrm{pen}(\mathcal{P}_i,\mathcal{P}_j)$ otherwise. Let $\max_x(\mathcal{P}_i) = \max_{u \in \mathcal{P}_i} u_x$ and $\min_x(\mathcal{P}_i) = \min_{u \in \mathcal{P}_i} u_x$ denote the bounding coordinates of $\mathcal{P}_i$ along the x-axis (analogously for y and z axes). The geometric atomic predicates are defined as follows:

\begin{itemize}
\item $\CloseE(i,j) : \ \ \mathrm{dist}(\mathcal{P}_i,\mathcal{P}_j) \le \varepsilon_c$
\item $\FarE(i,j) : \ \ \mathrm{dist}(\mathcal{P}_i,\mathcal{P}_j) \ge \varepsilon_f$
\item $\Touch(i,j) : \ \ |\mathrm{clr}_{ij}| \le \varepsilon$
\item $\Ovlp(i,j) : \ \ \mathrm{clr}_{ij} < -\dov$
\item $\PartOvlp(i,j) : \ \ \Ovlp(i,j) \wedge \neg\EnclIn(i,j) \wedge \neg\EnclIn(j,i)$
\item $\EnclIn(i,j) : \ \ \mathcal{P}_i \subset (\mathcal{P}_j \ominus \mathcal{B}_{\din})$
\item $\LeftOf(i,j) : \ \ \max_x(\mathcal{P}_i) + \kappa \le \min_x(\mathcal{P}_j)$
\item $\RightOf(i,j) : \ \ \max_x(\mathcal{P}_j) + \kappa \le \min_x(\mathcal{P}_i)$
\item $\mathrm{Behind}(i,j) : \ \ \max_y(\mathcal{P}_i) + \kappa \le \min_y(\mathcal{P}_j)$
\item $\mathrm{InFrontOf}(i,j) : \ \ \max_y(\mathcal{P}_j) + \kappa \le \min_y(\mathcal{P}_i)$
\item $\Below(i,j) : \ \ \max_z(\mathcal{P}_i) + \kappa \le \min_z(\mathcal{P}_j)$
\item $\Above(i,j) : \ \ \max_z(\mathcal{P}_j) + \kappa \le \min_z(\mathcal{P}_i)$
\item $\Between_{\mathrm{px}}(a,b,c) : \ \ \max_x(\mathcal{P}_a) + \kappa \le \min_x(\mathcal{P}_b) \ \wedge\ \max_x(\mathcal{P}_b) + \kappa \le \min_x(\mathcal{P}_c)$
\item $\Between_{\mathrm{py}}(a,b,c) : \ \ \max_y(\mathcal{P}_a) + \kappa \le \min_y(\mathcal{P}_b) \ \wedge\ \max_y(\mathcal{P}_b) + \kappa \le \min_y(\mathcal{P}_c)$
\item $\text{oriented}(i,j;\kappa) : \ \ \mathrm{ecd}(u_i,u_j) \le \kappa$
\end{itemize}
where $\ominus$ denotes set erosion (Pontryagin difference), i.e., $\mathcal{A} \ominus \mathcal{B} = \{x \mid x + \mathcal{B} \subseteq \mathcal{A}\}$, $\mathcal{B}_{\din}$ is a ball of radius $\din$, $u_i$ is the unit heading of object $i$, and $\mathrm{ecd}(u_i,u_j)=\tfrac{1}{2}\|u_i-u_j\|_2^2$.
\end{definition}

\begin{definition}[Quantitative Semantics]
Let $\Rob(s_t,\cdot)$ denote the exact robustness of a geometric atomic predicate under scene state $s_t$. The quantitative semantics are:

\begin{itemize}
\item $\Rob(\CloseE(i,j)) = \varepsilon_c - \mathrm{dist}(\mathcal{P}_i,\mathcal{P}_j)$
\item $\Rob(\FarE(i,j)) = \mathrm{dist}(\mathcal{P}_i,\mathcal{P}_j) - \varepsilon_f$
\item $\Rob(\Touch(i,j)) = -|\mathrm{clr}_{ij}| + \varepsilon$
\item $\Rob(\Ovlp(i,j)) = -\mathrm{clr}_{ij} - \dov$
\item $\begin{aligned}[t] 
      &\Rob(\PartOvlp(i,j)) = \min(\Rob(\Ovlp(i,j)), \\&-\Rob(\EnclIn(i,j)), 
      -\Rob(\EnclIn(j,i))) 
      \end{aligned}$
\item $\Rob(\EnclIn(i,j)) = -\din - \max_{u \in \mathcal{P}_i} \mathrm{sd}(u, \partial \mathcal{P}_j)$
\item $\Rob(\LeftOf(i,j)) = \min_x(\mathcal{P}_j) - \max_x(\mathcal{P}_i) - \kappa$
\item $\Rob(\RightOf(i,j)) = \min_x(\mathcal{P}_i) - \max_x(\mathcal{P}_j) - \kappa$
\item $\Rob(\mathrm{Behind}(i,j)) = \min_y(\mathcal{P}_j) - \max_y(\mathcal{P}_i) - \kappa$
\item $\Rob(\mathrm{InFrontOf}(i,j)) = \min_y(\mathcal{P}_i) - \max_y(\mathcal{P}_j) - \kappa$
\item $\Rob(\Below(i,j)) = \min_z(\mathcal{P}_j) - \max_z(\mathcal{P}_i) - \kappa$
\item $\Rob(\Above(i,j)) = \min_z(\mathcal{P}_i) - \max_z(\mathcal{P}_j) - \kappa$
\item $\Rob(\Between_{\mathrm{px}}(a,b,c)) = \min( \min_x(\mathcal{P}_b) - \max_x(\mathcal{P}_a) - \kappa, \min_x(\mathcal{P}_c) - \max_x(\mathcal{P}_b) - \kappa )$
\item $\Rob(\Between_{\mathrm{py}}(a,b,c)) = \min( \min_y(\mathcal{P}_b) - \max_y(\mathcal{P}_a) - \kappa, \min_y(\mathcal{P}_c) - \max_y(\mathcal{P}_b) - \kappa )$
\item $\Rob(\text{oriented}(i,j;\kappa)) = \kappa - \tfrac{1}{2}\|u_i-u_j\|_2^2$

\end{itemize}
where $\mathrm{sd}(u, \partial \mathcal{P}_j)$ denotes the signed distance from point $u$ to the boundary of $\mathcal{P}_j$ (negative if $u$ is strictly inside $\mathcal{P}_j$).
\end{definition}

% ==========================================================

\subsection{Differentiable Problem Formulation}
\label{subsec:problem}

Given the scene trajectory $\xi$, let $\phi$ be a spatio-temporal specification formed by composing STL temporal operators with SpaTiaL spatial predicates. The corresponding exact robustness at time $0$ is denoted by $\Rob(\xi,\phi,0)$.

The central objective of this paper is to construct a differentiable surrogate robustness
\(
\sRob(\xi,\phi,0)
\)
for SpaTiaL-based specifications, such that it remains semantically faithful to \(\Rob(\xi,\phi,0)\) while enabling end-to-end gradient propagation with respect to geometric states. Specifically, we seek a formulation such that \(\sRob(\xi,\phi,0)\approx \Rob(\xi,\phi,0)\), and \(\frac{\partial \sRob(\xi,\phi,0)}{\partial s_t}\) exists and remains informative for all \(t\in\{0,\dots,T\}\).

The difficulty is not the temporal composition itself, since STL robustness is recursively defined through \(\min/\max\)-type operators, but the spatial atomic robustness terms in SpaTiaL (e.g., distance, overlap, enclosure, and directional relations). In standard implementations, these spatial quantities are evaluated by discrete geometry engines with branching logic and boolean operations, which break the computational graph and make gradients undefined or uninformative.

Therefore, the problem addressed in this work is to reformulate SpaTiaL spatial semantics over polygonal objects into a fully tensorized, smooth, and autograd-compatible computation graph. Once the spatial atomic predicates become differentiable, they can be composed with smoothed temporal operators to produce an end-to-end differentiable robustness pipeline, which directly supports gradient-based trajectory optimization and specification parameter learning.
\section{Methodology}
\label{sec:method}

To integrate SpaTiaL into continuous optimization, the entire evaluation pipeline must be differentiable. Here, we formulate \emph{Differentiable SpaTiaL}, a fully tensorized computational graph that analytically relaxes discrete geometric queries into smooth spatial semantics, as outlined in Fig.~\ref{fig:computation_graph_cd_single}

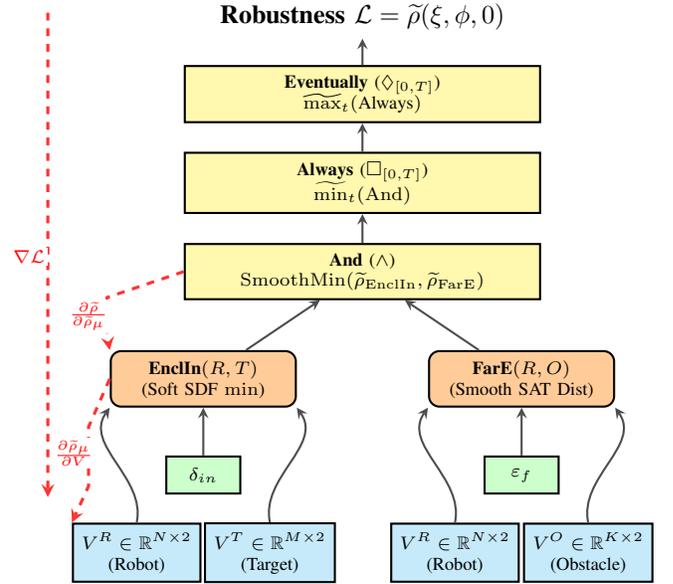
\begin{figure}[t]
\centering
\resizebox{0.99\columnwidth}{!}{%
\begin{tikzpicture}[
    >=stealth,
    inputbox/.style={rectangle, draw=black, fill=cyan!20, thick, minimum width=1.35cm, minimum height=0.55cm, align=center, font=\scriptsize},
    parambox/.style={rectangle, draw=black, fill=green!20, thick, minimum width=1.0cm, minimum height=0.5cm, align=center, font=\scriptsize},
    spatialbox/.style={rectangle, rounded corners, draw=black, fill=orange!40, thick, minimum width=2.5cm, minimum height=0.75cm, align=center, font=\scriptsize},
    logicbox/.style={rectangle, draw=black, fill=yellow!40, thick, minimum width=4.8cm, minimum height=0.72cm, align=center, font=\scriptsize},
    arrow/.style={->, thick, draw=black!70},
    gradarrow/.style={->, dashed, very thick, draw=red!80},
    physrob/.style={rectangle, rounded corners=2pt, draw=blue!80, fill=cyan!30, thick, minimum width=1.0cm, minimum height=0.9cm, align=center, font=\scriptsize},
    physobs/.style={circle, draw=black!80, fill=gray!30, thick, minimum size=1.35cm, align=center, font=\scriptsize},
    phystar/.style={rectangle, dashed, draw=green!60!black, fill=green!5, thick, minimum width=1.9cm, minimum height=1.9cm, align=center, font=\scriptsize}
]

% (c) Tensorized SpaTiaL Graph
\begin{scope}[xshift=0cm]
    \node[inputbox] (vr1) at (-1.8, 0) {$V^R \in \mathbb{R}^{N \times 2}$ \\ (Robot)};
    \node[inputbox] (vt)  at (0, 0) {$V^T \in \mathbb{R}^{M \times 2}$ \\ (Target)};
    \node[inputbox] (vr2) at (2.5, 0) {$V^R \in \mathbb{R}^{N \times 2}$ \\ (Robot)};
    \node[inputbox] (vo)  at (4.3, 0) {$V^O \in \mathbb{R}^{K \times 2}$ \\ (Obstacle)};

    \node[parambox] (din) at (-0.9, 1.05) {$\delta_{in}$};
    \node[parambox] (ef)  at (3.4, 1.05) {$\varepsilon_f$};

    \node[spatialbox] (enclin) at (-0.9, 2.35) {\textbf{EnclIn}$(R,T)$ \\ (Soft SDF $\min$)};
    \node[spatialbox] (fare)   at (3.4, 2.35) {\textbf{FarE}$(R,O)$ \\ (Smooth SAT Dist)};

    \draw[arrow] (vr1.north) to[out=78, in=235] (enclin.south west);
    \draw[arrow] (vt.north)  to[out=102, in=305] (enclin.south east);
    \draw[arrow] (din.north) -- (enclin.south);

    \draw[arrow] (vr2.north) to[out=78, in=235] (fare.south west);
    \draw[arrow] (vo.north)  to[out=102, in=305] (fare.south east);
    \draw[arrow] (ef.north) -- (fare.south);

    \node[logicbox] (and) at (1.25, 3.8) {\textbf{And} ($\wedge$) \\ $\mathrm{SmoothMin}(\sRob_{\mathrm{EnclIn}},\sRob_{\mathrm{FarE}})$};
    \node[logicbox] (always) at (1.25, 5.0) {\textbf{Always} ($\square_{[0,T]}$) \\ $\widetilde{\min}_{t}(\mathrm{And})$};
    \node[logicbox] (eventually) at (1.25, 6.2) {\textbf{Eventually} ($\Diamond_{[0,T]}$) \\ $\widetilde{\max}_{t}(\text{Always})$};

    \node (robustness) at (1.25, 7.25) {\normalsize \textbf{Robustness} $\mathcal{L}=\widetilde{\rho}(\xi,\phi,0)$};

    \draw[arrow] (enclin) -- (and);
    \draw[arrow] (fare) -- (and);
    \draw[arrow] (and) -- (always);
    \draw[arrow] (always) -- (eventually);
    \draw[arrow] (eventually) -- (robustness);

    \draw[gradarrow] (-3.0,7.3) -- (-3.0,0.75);
    \node[left, font=\scriptsize\color{red}, fill=white, inner sep=1pt] at (-3.0,4.0) {$\nabla \mathcal{L}$};

    \draw[gradarrow] (and.west) -- (-2.2,3.45) -- (-2.2,2.95) -- (enclin.north west);
    \node[font=\tiny\color{red}, fill=white, inner sep=1pt] at (-2.45,3.2) {$\frac{\partial \widetilde{\rho}}{\partial \sRob_{\mu}}$};

    \draw[gradarrow] (enclin.west) -- (-2.45,1.7) -- (-2.45,0.9) -- (vr1.north west);
    \node[font=\tiny\color{red}, fill=white, inner sep=1pt] at (-2.65,1.35) {$\frac{\partial \sRob_{\mu}}{\partial V}$};

    % \node[font=\small] at (1.25, -0.8) {(a) Tensorized SpaTiaL};
\end{scope}

% % (d) Physical Execution
% \begin{scope}[xshift=8.9cm, yshift=3.1cm]
%     % panel
%     \draw[thick, draw=black!50, fill=gray!5, rounded corners]
%         (-2.2, -3.05) rectangle (2.2, 3.05);

%     % target: compact, top-left
%     \node[
%         rectangle,
%         dashed,
%         draw=green!60!black,
%         fill=green!5,
%         thick,
%         minimum width=1.45cm,
%         minimum height=0.85cm,
%         align=center,
%         font=\scriptsize
%     ] (target) at (-1.20, 2.10) {Target ($V^T$)};

%     % obstacle: bottom-right
%     \node[physobs, minimum size=1.55cm] (obs) at (1.00, -1.75) {Obstacle\\($V^O$)};

%     % robot: near obstacle
%     \node[physrob, minimum width=1.18cm, minimum height=0.92cm] (robot) at (0.30, -0.20) {Robot\\($V^R$)};

%     % two red arrows, both toward upper-left
%     \draw[gradarrow, very thick] (robot.north west) -- (-0.35, 0.85);
%     \draw[gradarrow, very thick] (robot.west) -- (-0.55, 0.15);

%     % labels: no white background
%     \node[
%         font=\scriptsize\color{red},
%         inner sep=1pt,
%         align=left
%     ] at (0.35, 1.00)
%         {$\frac{\partial \widetilde{\rho}_{\mathrm{EnclIn}}}{\partial V^R}$\\(pull)};

%     \node[
%         font=\scriptsize\color{red},
%         inner sep=1pt,
%         align=left
%     ] at (-1.05, 0.23)
%         {$\frac{\partial \widetilde{\rho}_{\mathrm{FarE}}}{\partial V^R}$\\(push)};

%     % \node[font=\small] at (0, -3.55) {(b) Physical Execution};
% \end{scope}
\end{tikzpicture}%
}
\caption{Smooth spatial and temporal operators form a differentiable robustness computation graph.}
\label{fig:computation_graph_cd_single}
% \vspace{-2em}
\end{figure}
\subsection{Overview of Tensorized Spatial Semantics}
Consider the trajectory $\xi$ where each state $s_t$ is represented natively as multi-dimensional tensors (e.g., vertex coordinates of bounding polygons). To optimize $\xi$ with respect to a high-level spatio-temporal formula $\phi$ via gradient descent, we define a loss function $\mathcal{L}(\sRob(\xi, \phi, 0))$ over the quantitative smoothed logical robustness $\sRob$, and we must compute the analytic gradient $\nabla_{\xi} \mathcal{L}$.

Applying the chain rule, the gradient flow to the geometric state at time $t$ is decomposed as:
\begin{equation} \label{eq:chain_rule}
    \frac{\partial \mathcal{L}}{\partial s_t} = \frac{\partial \mathcal{L}}{\partial \sRob} \sum_{k} \left( \frac{\partial \sRob}{\partial \sRob(\mu_k)} \frac{\partial \sRob(\mu_k)}{\partial s_t} \right),
\end{equation}
where $\sRob(\mu_k)$ denotes the smoothed robustness of the $k$-th spatial atomic predicate (e.g., distance, overlap) evaluated at time $t$. Existing differentiable temporal logics \cite{leung2021backpropagationsignaltemporallogic} provide the temporal gradient mapping $\frac{\partial \sRob}{\partial \sRob(\mu_k)}$ by smoothing operators like \textit{Always} and \textit{Eventually}. However, standard geometry engines (e.g., Shapely) render the spatial gradient $\frac{\partial \sRob(\mu_k)}{\partial s_t}$ undefined or zero almost everywhere due to discrete Boolean operations. 

\emph{Differentiable SpaTiaL} bridges this gap by mathematically formulating $\frac{\partial \sRob(\mu_k)}{\partial s_t}$ natively in tensor space. We circumvent non-differentiable bottlenecks by relaxing exact geometric queries using the LogSumExp (LSE) function as a smooth maximum approximator:
\begin{equation}
    \widetilde{\max}_{\tau}(\mathbf{x}) = \tau \log \sum\nolimits_{i=1}^{N} \exp(x_i / \tau),
\end{equation}
\begin{equation}
    \widetilde{\min}_{\tau}(\mathbf{x}) = -\widetilde{\max}_{\tau}(-\mathbf{x}),
\end{equation}
where $\tau > 0$ controls the smoothness temperature. As $\tau \to 0$, the function approaches the exact $\max$/$\min$ operators, while $\tau > 0$ ensures a smooth, non-zero gradient landscape across the entire spatial domain.

\subsection{Differentiable Penetration via Smooth SAT}
Standard penetration depth algorithms (e.g., EPA) involve discrete iterative searches over simplices, which fracture the computational graph. Instead, we derive a smooth, fully tensorized variant of the Separating Axis Theorem (SAT) suitable for convex polygonal sets, as illustrated in Fig.~\ref{fig:smooth_sat}.

\begin{figure}[ht]
    \centering
    \includegraphics[width=\columnwidth]{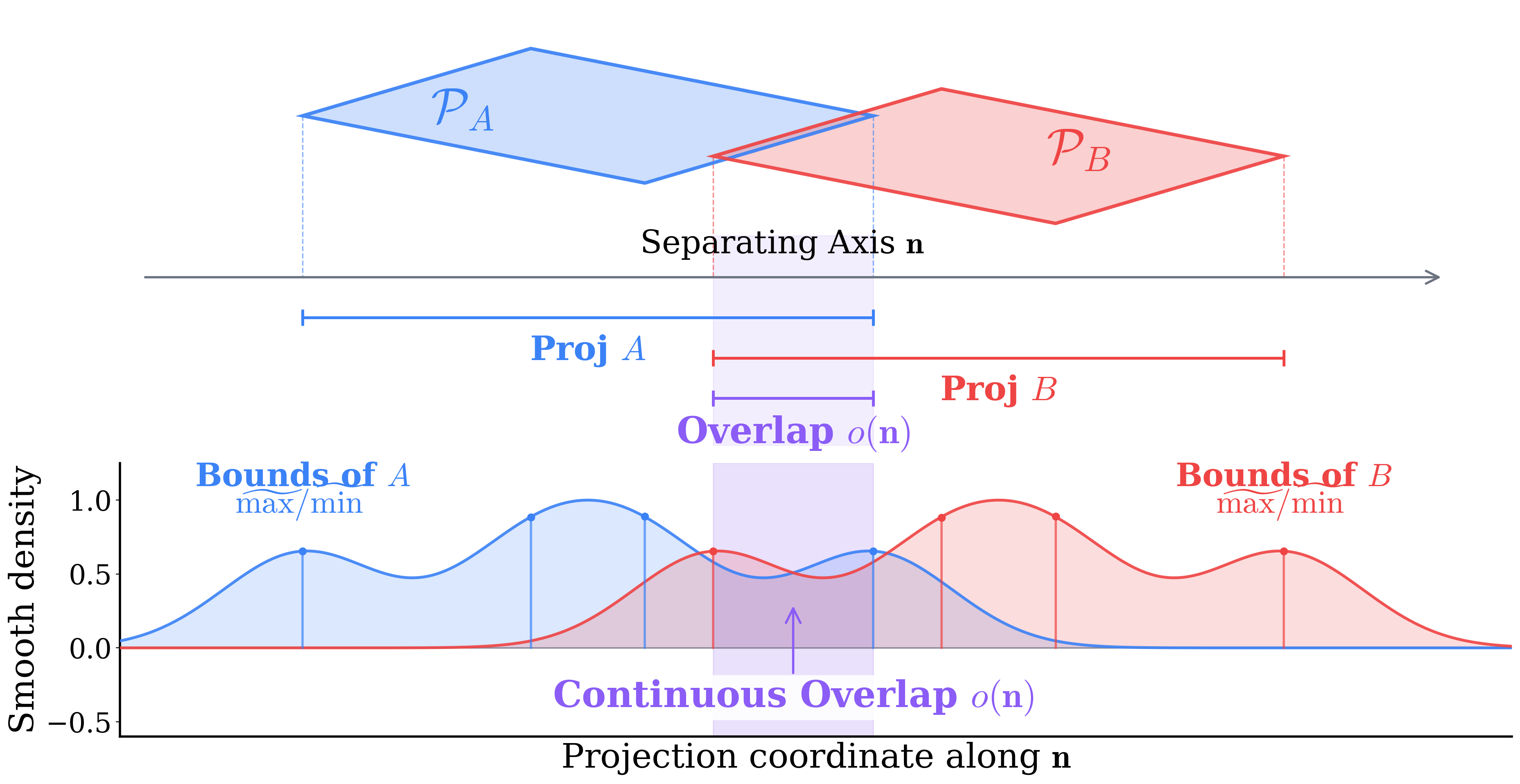} 
    \caption{Differentiable penetration depth via Smooth SAT. Exact discrete overlap is relaxed into a smooth, differentiable repulsion field.}
    \label{fig:smooth_sat}
    % \vspace{-1em}
\end{figure}

Let $\mathcal{P}_A$ and $\mathcal{P}_B$ be two convex polygons parameterized by their counter-clockwise vertex tensors $\mathbf{V}^A \in \mathbb{R}^{N_A \times d}$ and $\mathbf{V}^B \in \mathbb{R}^{N_B \times d}$. Let $\mathcal{N}$ be the combined set of all inward unit edge normals from both polygons. For a specific projection axis $\mathbf{n} \in \mathcal{N}$, the projection of polygon $\mathcal{P}_A$ yields a continuous range $[\min(\mathbf{V}^A \mathbf{n}), \max(\mathbf{V}^A \mathbf{n})]$. Using our smooth operators, the upper and lower projection bounds are defined as $P_{\max}^A = \widetilde{\max}_{\tau}(\mathbf{V}^A \mathbf{n})$ and $P_{\min}^A = \widetilde{\min}_{\tau}(\mathbf{V}^A \mathbf{n})$.

The overlapping margin on the axis $\mathbf{n}$ is computed as:
\begin{equation}
    o(\mathbf{n}) = \widetilde{\min}_{\tau}(P_{\max}^A, P_{\max}^B) - \widetilde{\max}_{\tau}(P_{\min}^A, P_{\min}^B).
\end{equation}
The differentiable penetration depth across all separating axes is then formulated by taking the smooth minimum of the overlaps:
\begin{equation} \label{eq:sat_pen}
    pen(\mathcal{P}_A, \mathcal{P}_B) = \mathrm{ReLU} \Bigl( \widetilde{\min}_{\mathbf{n} \in \mathcal{N}} \bigl( o(\mathbf{n}) \bigr) \Bigr).
\end{equation}
This formulation yields $pen > 0$ when the polygons intersect, generating a smooth repulsion gradient analytically derived from dense matrix multiplications, entirely avoiding discrete branching.

\begin{figure}[ht]
    \centering
    \begin{tikzpicture}[scale=0.82, font=\small, >=stealth]
        % SDF Contours around Polygon B (shifted further left)
        \draw[thick, blue!20, rounded corners=30pt, fill=blue!5]
            (-0.65, 2.0) -- (-2.65, 0) -- (-0.65, -2.0) -- (1.35, 0) -- cycle;
        \draw[thick, blue!40, rounded corners=20pt, fill=blue!10]
            (-0.65, 1.6) -- (-2.25, 0) -- (-0.65, -1.6) -- (0.95, 0) -- cycle;
        \draw[thick, blue!60, rounded corners=10pt, fill=blue!20]
            (-0.65, 1.2) -- (-1.85, 0) -- (-0.65, -1.2) -- (0.55, 0) -- cycle;

        % Center polygon B (Obstacle)
        \draw[thick, fill=red!30, draw=red!80]
            (-0.65, 0.8) -- (-1.45, 0) -- (-0.65, -0.8) -- (0.15, 0) -- cycle;
        \node[red!80!black] at (-0.65,0) {$\mathcal{P}_B$};

        % Robot Polygon A (Sampled Geometry)
        \draw[thick, fill=green!20, draw=green!60!black]
            (2.4, 1.0) -- (3.4, 1.5) -- (3.9, 0.5) -- (2.9, 0.0) -- cycle;
        \node[green!50!black] at (3.25, 0.8) {$\mathcal{P}_A$};

        % Boundary sampling points P^A
        \foreach \p in {
            (2.4, 1.0), (2.9, 1.25), (3.4, 1.5),
            (3.65, 1.0), (3.9, 0.5),
            (3.4, 0.25), (2.9, 0.0),
            (2.65, 0.5)
        } {
            \fill[black] \p circle (1.5pt);
        }

        % Highlighted sample point 'p'
        \fill[orange] (2.4, 1.0) circle (2pt);
        \node[black] at (2.58, 1.18) {$p$};

        % Distance vector to P_B
        \draw[->, thick, dashed, orange] (2.4, 1.0) -- (-0.05, 0.45);
        \node[black] at (1.28, 1.02) {$sd(p,\mathcal{P}_B)$};

        % Labels / legend
        \node[blue!80] at (-1.75, 2.28) {SDF Contours ($m_{out}$)};
        \node[black, align=center] at (4.10, 1.65) {Sampled Boundary\\$\mathbf{P}^A$};
        \draw[->, thin, gray] (3.98, 1.40) .. controls (3.85, 1.15) .. (3.68, 0.95);

    \end{tikzpicture}
    \caption{Differentiable Signed Distance Field (SDF) via boundary sampling. Soft aggregation ensures informative gradients in both colliding and non-colliding states.}
    \label{fig:soft_sdf}
    % \vspace{-1em}
\end{figure}
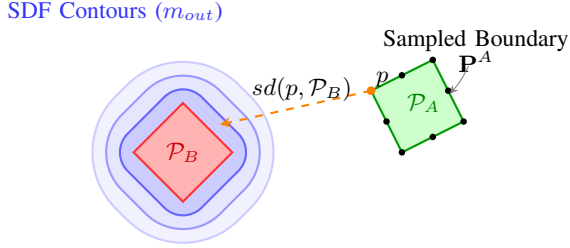

\subsection{Smooth Approximations of Geometric Distances}
While the SAT penetration depth handles overlapping states, evaluating strictly separated states requires a differentiable distance metric. As shown in Fig.~\ref{fig:soft_sdf}, we approximate the exact boundary-to-boundary shortest path using soft boundary sampling and smooth half-space intersections.
For brevity, we write $d(p,\mathcal{P}) := m_{out}(p)$ for the smooth unsigned point-to-polygon distance, use $\mathrm{dist}(\mathcal{P}_A,\mathcal{P}_B)$ for polygon-to-polygon distance, and use $\mathrm{sd}$ generically for signed-distance quantities; the intended meaning is disambiguated by the argument types.

We sample a set of points $\mathbf{P}^A$ uniformly along the perimeter of $\mathcal{P}_A$. For any sampled point $p \in \mathbf{P}^A$, its spatial relationship to polygon $\mathcal{P}_B$ is evaluated via half-space constraints. Let $e_i^B$ denote the $i$-th edge of $\mathcal{P}_B$, defined by a vertex $v_i^B \in e_i^B$ and an inward unit normal $\mathbf{n}_i^B$. We let $s_i = (p - v_i^B) \cdot \mathbf{n}_i^B$ denote the signed distance from $p$ to the $i$-th edge of $\mathcal{P}_B$. The internal margin is given by $m_{in}(p) = \widetilde{\min}_{\tau}(\{s_i\}_{i=1}^{N_B})$. To capture the external distance, we compute the smooth point-to-segment distance $d_{seg}(p, e_i^B)$ for all edges $e_i^B$, yielding the external margin $m_{out}(p) = \widetilde{\min}_{\tau}(\{d_{seg}\})$.
The smooth signed distance for a single point $p$ to polygon $\mathcal{P}_B$ is smoothly blended using a sigmoid weighting function $w(p)=\sigma(k \cdot m_{in}(p))$, where $\sigma(x) = 1 / (1 + \exp(-x))$ and $k$ is a scaling constant:
\begin{equation}
    sd(p, \mathcal{P}_B) = (1 - w(p)) \cdot m_{out}(p) - w(p) \cdot m_{in}(p).
\end{equation}
A negative value strictly indicates that $p$ is enclosed within $\mathcal{P}_B$. The symmetric unsigned distance between the two polygons is aggregated over all boundary samples:
\begin{equation}\label{eq:smooth_dist}
dist(\mathcal{P}_A,\mathcal{P}_B)
\approx
\mathop{\widetilde{\min}}\limits_{\tau}
{(
\mathop{\widetilde{\min}}\limits_{p\in\mathbf{P}^A} d(p,\mathcal{P}_B),\;
\mathop{\widetilde{\min}}\limits_{p\in\mathbf{P}^B} d(p,\mathcal{P}_A)
)}
\end{equation}

By unifying the smooth SAT penetration (Eq. \ref{eq:sat_pen}) and the soft boundary distance (Eq. \ref{eq:smooth_dist}), we define a universally differentiable signed distance $sd(\mathcal{P}_A, \mathcal{P}_B)$. This metric serves as the quantitative smoothed robustness value $q$ for distance-based spatial predicates (e.g., $\CloseE, \FarE$).

\subsection{Compositional Spatial Predicates}
Leveraging the tensorized spatial primitives, higher-level topological and directional relations in the SpaTiaL grammar are synthesized purely through differentiable algebraic operations, as depicted in Fig.~\ref{fig:predicates}.

\begin{figure}[ht]
    \centering
    \includegraphics[width=\columnwidth]{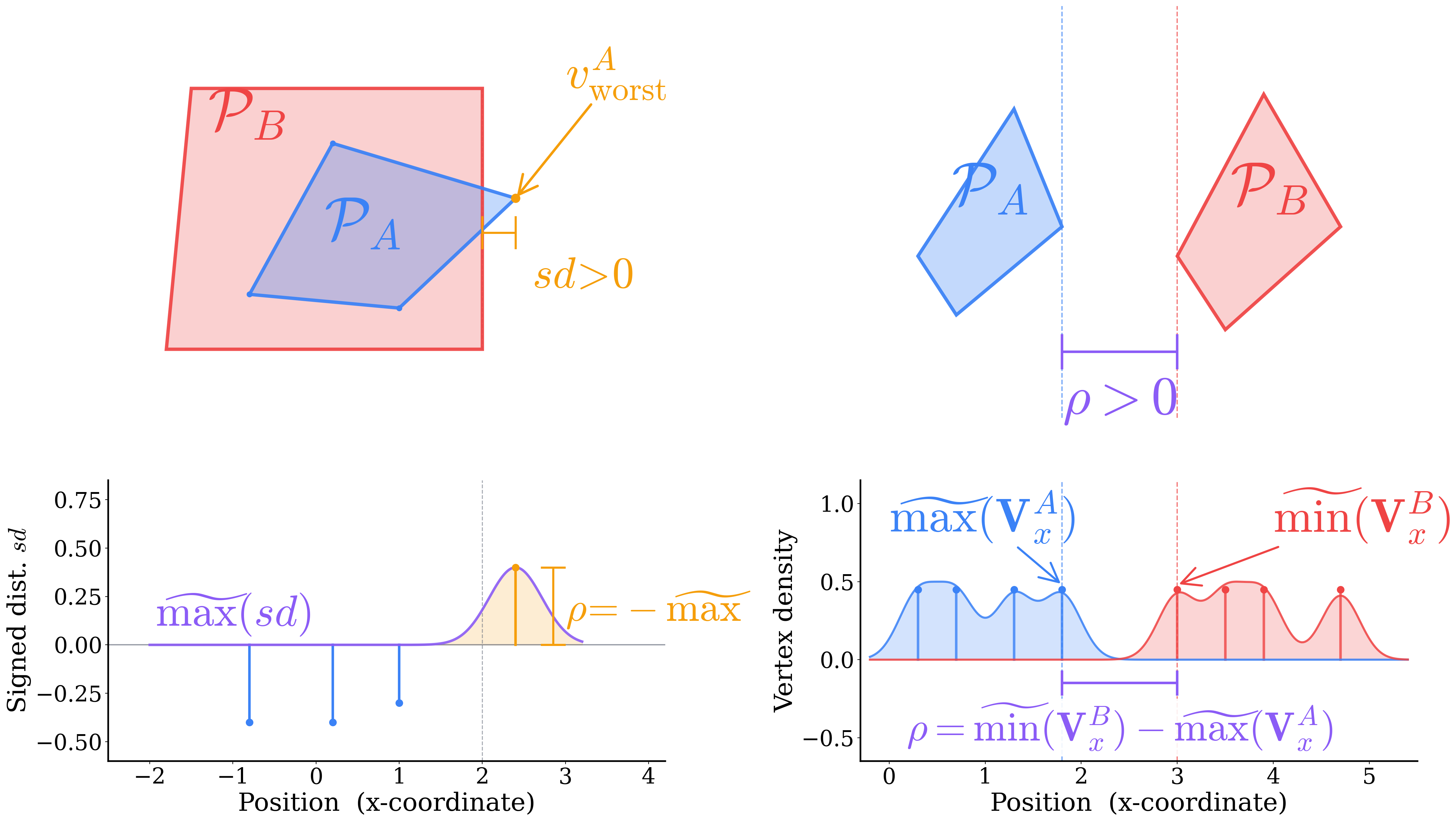} 
    \caption{Geometric interpretation of compositional predicates (left: EnclIn, right: leftOf). Spatial relations are resolved as differentiable algebraic operations over vertex tensors.}
    \label{fig:predicates}
    % \vspace{-2em}
\end{figure}
\subsubsection{Enclosure and Containment}
The predicate $\EnclIn(\mathcal{P}_A, \mathcal{P}_B)$ requires all parts of $\mathcal{P}_A$ to reside strictly within $\mathcal{P}_B$. Utilizing the point-to-polygon signed distance, the smooth atomic robustness $q$ is computed by extracting the least enclosed (most outward) point of $\mathcal{P}_A$:
\begin{equation}
    \sRob(\EnclIn(A,B)) = -\delta_{in}- \widetilde{\max}_{\tau} ( \{ sd(v_i^A, \mathcal{P}_B) \mid v_i^A \in \mathbf{V}^A \} ).
\end{equation}
A robustness $\sRob \ge 0$ indicates strict geometric containment.

\subsubsection{Directional Projections}
Predicates evaluating relative poses, such as $\LeftOf$ or $\Above$, are formulated by projecting the geometric centroids or bounding vertices onto the global coordinate axes. Let $c_A$ and $c_B$ be the differentiable centroids of the polygons. The smooth atomic robustness for $\LeftOf$ is straightforwardly defined via smooth coordinate comparisons:
\begin{equation}
    \sRob(\LeftOf(A,B)) = \widetilde{\min}_{\tau} (\mathbf{V}^B_x) - \widetilde{\max}_{\tau} (\mathbf{V}^A_x)-\kappa,
\end{equation}
where $\mathbf{V}_x$ denotes the x-coordinate tensor of the vertices. 

\subsubsection{Relative Bearing}
For orientation-dependent tasks, the predicate evaluates the bearing from object $A$ to object $B$. The relative angle is obtained via the differentiable arc-tangent function applied to the centroid displacements:
\(
    \sRob(\mathrm{BearingTo}(A,B, \theta_{ref})) = \kappa - \left\| \mathrm{atan2}(c_B^y - c_A^y, c_B^x - c_A^x) - \theta_{ref} \right\|_2^2,
\)
where $\kappa$ is the tolerance threshold. By executing these compositional predicates as vectorized tensor operations, the framework enables complex spatial conditions to be efficiently evaluated and differentiated across massive batches of state trajectories.

The constructions above establish differentiable implementations of representative compositional predicates. Beyond differentiability, however, we also require a basic semantic guarantee: the smooth relaxation should not artificially overestimate predicate satisfaction. In particular, for safety-critical optimization, it is desirable that positive smoothed robustness implies positive exact robustness (i.e., a conservative approximation). We next show that this property follows from standard LogSumExp (LSE) bounds for the core predicate primitives used in our framework.

% For the core predicates considered in our implementation, the smoothed spatial robustness reduces to one of two primitives applied to convex polygons with exact vertices:

% \begin{enumerate}[label=(\roman*)]
% \item a \emph{minimum} over scalar quantities
%   (distance, signed projection)---used by $\FarE$ and directional predicates ($\LeftOf$, $\RightOf$, $\Above$,
%   $\Below$, $\Behind$, $\InFront$);
% \item a \emph{negated maximum} of signed distances---used by
%   $\EnclIn$.
% \end{enumerate}

\begin{lemma}
\label{lemma:lse_bounds}
For any vector $\mathbf{x} \in \mathbb{R}^N$ and temperature
$\tau > 0$:
\begin{align}
    \max(\mathbf{x})
      &\le \smax(\mathbf{x})
      \le \max(\mathbf{x}) + \tau \log N,
      \label{eq:max_bound} \\
    \min(\mathbf{x}) - \tau \log N
      &\le \smin(\mathbf{x})
      \le \min(\mathbf{x}).
      \label{eq:min_bound}
\end{align}

\end{lemma}

\noindent Here $\smin$ and $\smax$ denote the LSE-based smooth relaxations of $\min$ and $\max$, respectively. By Lemma~\ref{lemma:lse_bounds}, we directly obtain
\(
  \smin(\mathbf{x}) \le \min(\mathbf{x}), \qquad
  \smax(\mathbf{x}) \ge \max(\mathbf{x}).
  \label{eq:short_bounds}
\)

% \begin{theorem}
% \label{thm:short}
% For convex polygons with exact vertex representations and $\tau>0$, any spatial predicate whose robustness is implemented as either (i) a minimum over exact scalar quantities or (ii) a negated maximum over exact signed-distance quantities satisfies
% \(
% \sRob(\phi) \le \rho_{\mathrm{exact}}(\phi).
% \)
% \end{theorem}

% \begin{proof}
% \textbf{(i) Minimum-based predicates ($\FarE$, directional).}
% The exact robustness has the form
% $\rho_{\mathrm{exact}} = \min_i f_i$
% for exact scalar quantities $f_i$. Replacing $\min$ with $\smin$ yields the smooth spatial robustness $q$:
% \begin{equation}
%   \sRob = \smin_i f_i
%     \;\overset{\eqref{eq:short_min}}{\le}\;
%     \min_i f_i = \rho_{\mathrm{exact}}.
% \end{equation}

% \noindent\textbf{(ii) Enclosure predicate ($\EnclIn$).}
% The exact robustness has the form
% $\rho_{\mathrm{exact}} = -\max_i g_i$.
% Replacing $\max$ with $\smax$ gives:
% \begin{equation}
%   \sRob = -\smax_i g_i
%     \;\overset{\eqref{eq:short_max}}{\le}\;
%     -\max_i g_i = \rho_{\mathrm{exact}}.
% \end{equation}
% \noindent Therefore, in both cases,
% \(
% \sRob(\phi) \le \rho_{\mathrm{exact}}(\phi)
% \),
% and in particular
% \(
% \sRob(\phi) > 0 \implies \rho_{\mathrm{exact}}(\phi) > 0
% \).
% Since conjunction ($\wedge$) and smoothed temporal operators ($\Box,\Diamond$) are also implemented via the same $\smin/\smax$ relaxations, the inequality is preserved compositionally, extending the conservative guarantee to compound specifications built from these predicate primitives.
% \end{proof}

\begin{theorem}
\label{thm:short}
We assume throughout that the specification is in positive normal form, without significant loss of generality in our setting.
For convex polygons with exact vertex representations and $\tau>0$:
\begin{enumerate}[label=(\roman*)]
\item Any spatial predicate whose robustness is a single-level $\smin$ or negated $\smax$ over exact scalar quantities (directional predicates) satisfies $\sRob(\phi) \le \rho_{\mathrm{exact}}(\phi)$.
\item For predicates involving boundary sampling or nested smooth operators ($\FarE$, $\CloseE$, $\EnclIn$, $\Ovlp$, $\Touch$), $\sRob(\phi) \le \rho_{\mathrm{exact}}(\phi) + \delta$, where $\delta = O(\tau \log N + h)$ depends on the LSE temperature $\tau$, polygon complexity $N$, and sampling spacing $h$. Both error sources vanish as $\tau \to 0$ and $h \to 0$.
\item $\textsc{Oriented}$ is exact: $\sRob=\rho_{\mathrm{exact}}$.
\end{enumerate}
\end{theorem}

\begin{proof}
\emph{Case~(i):} Directional predicates are sums of extremal coordinate terms with $\min/\max$ replaced by $\smin/\smax$. By Lemma~\ref{lemma:lse_bounds}, each smooth extremum is conservative, so $\sRob(\phi)\le \rho_{\mathrm{exact}}(\phi)$ and $\sRob(\phi)>0 \implies \rho_{\mathrm{exact}}(\phi)>0$.

\emph{Case~(ii):} For $\EnclIn$, the outer term satisfies $-\smax_i g_i \le -\max_i g_i$, but each $g_i=\mathrm{sd}(v_i^A,\mathcal{P}_B)$ is itself smoothed, contributing $O(\tau\log N_B)$ error. For $\FarE$, $\CloseE$, $\Ovlp$, $\Touch$, boundary sampling contributes $O(h)$ and nested smooth aggregation contributes $O(\tau\sum_{k=1}^{L}\log N_k)$. Hence
$\sRob(\phi)\le \rho_{\mathrm{exact}}(\phi)+\delta$ with $\delta \le Ch+\tau\sum_{k=1}^{L}\log N_k$ and $h=0$ for $\EnclIn$. Thus $\sRob(\phi)-\delta>0 \implies \rho_{\mathrm{exact}}(\phi)>0$, and $\delta\to 0$ as $\tau,h\to 0$.

\emph{Case~(iii):} $\textsc{Oriented}$ is exact, so $\sRob(\phi)=\rho_{\mathrm{exact}}(\phi)$.
\end{proof}

% This result establishes a conservative smoothing property for the corresponding spatial predicates: the smooth robustness may underestimate the exact robustness, but it does not create false-positive satisfaction. This property is particularly useful in gradient-based optimization, where $\sRob$ serves as the differentiable objective while retaining a safety-oriented semantic bias.

This result establishes a conservative smoothing property. For single-level predicates (Cases~\textit{i},~\textit{iii}), the smooth robustness does not overestimate the exact value; in particular, Case~\textit{i} is conservative and Case~\textit{iii} is exact, thereby precluding false positives. For Case~\textit{ii}, a bounded overestimation $\delta$ may arise, but it vanishes as $\tau, h \to 0$, preserving a safety-oriented semantic bias in practice.

\begin{figure*}[t]
    \centering
    \includegraphics[width=0.24\textwidth]{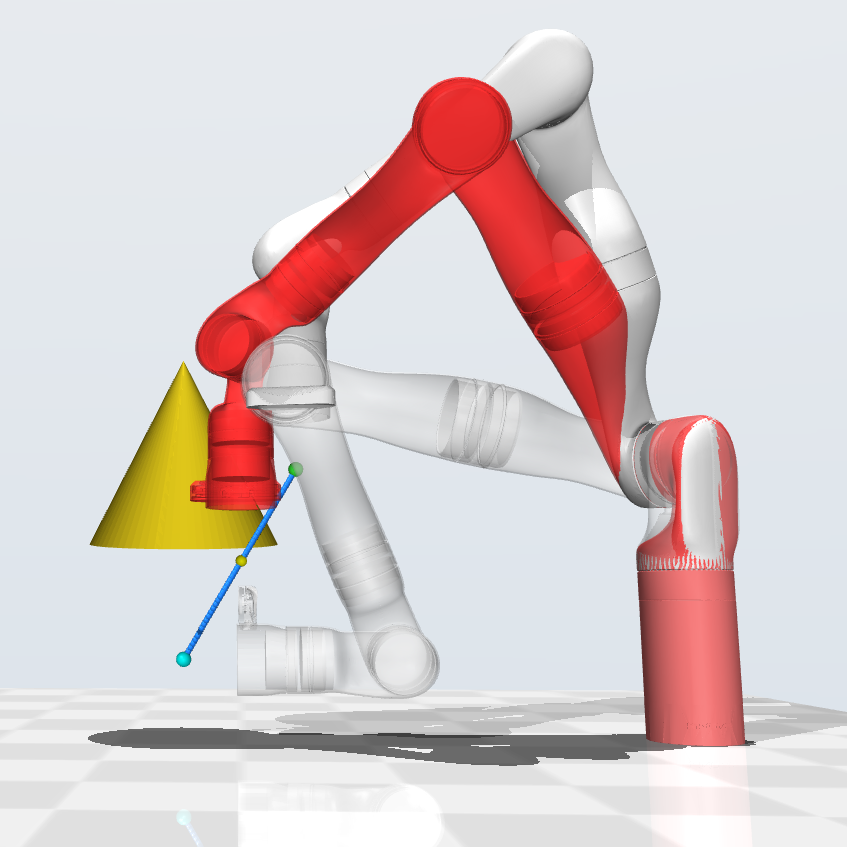}
    \hfill
    \includegraphics[width=0.24\textwidth]{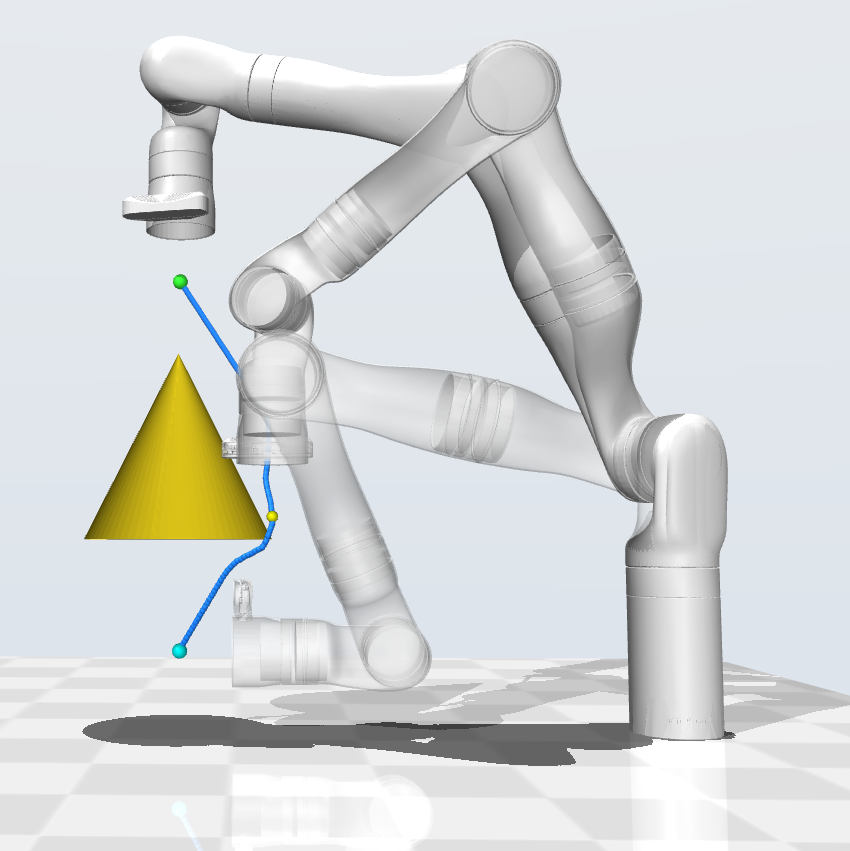}
    \hfill
    \includegraphics[width=0.24\textwidth]{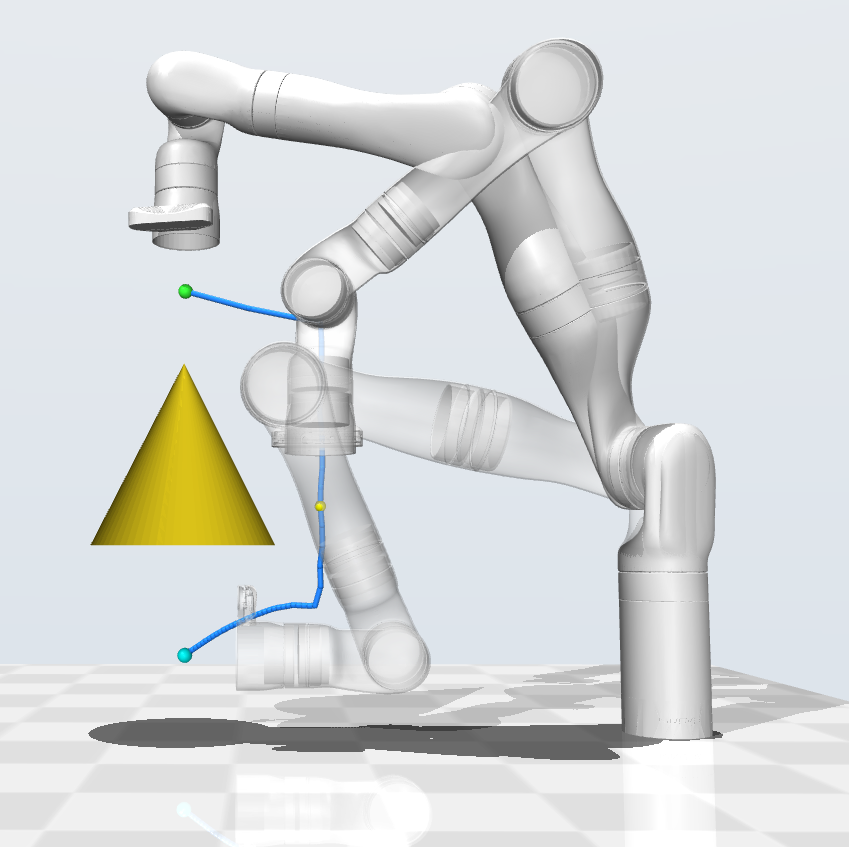}
    \hfill
    \includegraphics[width=0.26\textwidth]{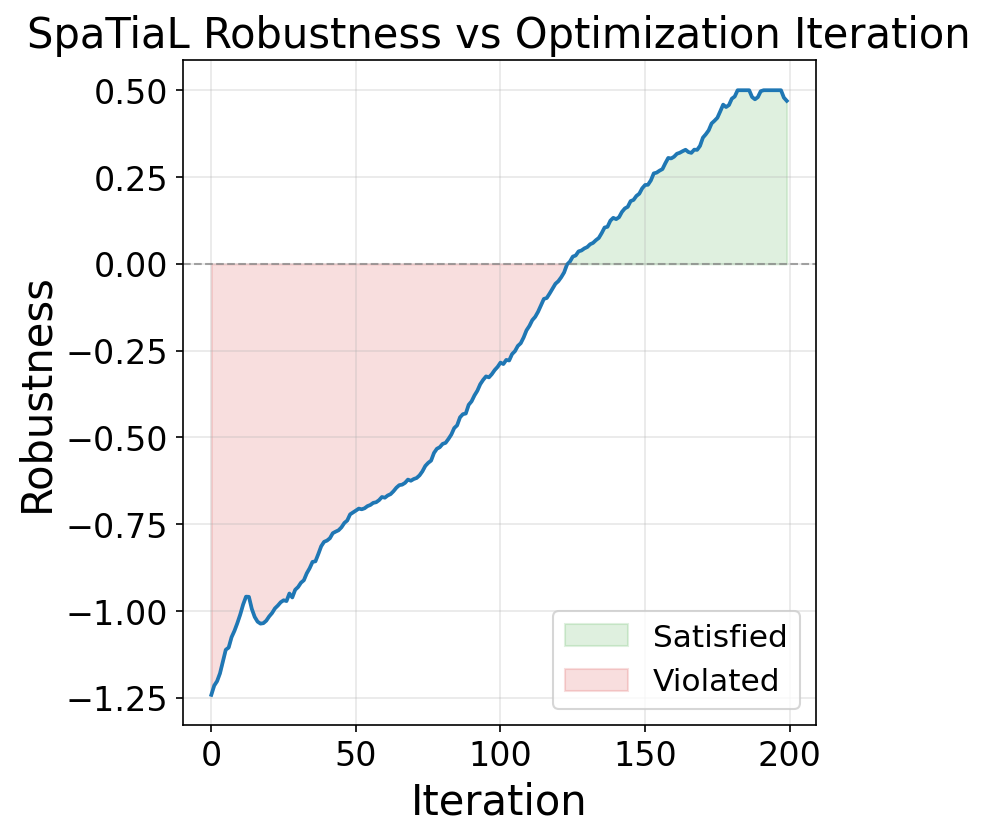}
    \caption{Qualitative spatio-temporal trajectory optimization using differentiable spatial semantics. From left to right: initial trajectory (violates), iteration 30, final trajectory (satisfies), and robustness analysis.}
    \label{fig:exp_opt_qual}
    % \vspace{-1.5em}
\end{figure*}

\subsection{Robotic Applications}
By unifying the temporal and spatial domains into a single end-to-end differentiable computation graph, our framework unlocks gradient-based methods for two primary robotic applications:

\subsubsection{Trajectory Optimization}
Under our framework, the robustness becomes a differentiable functional of the trajectory $\sRob(\xi,\phi,0) = F_{\phi}(s_0,\ldots,s_T),$
where $F_{\phi}$ is a smooth computation graph induced by the logical structure of $\phi$. 
Trajectory synthesis can therefore be formulated as
$\max_{\xi} \; \sRob(\xi,\phi,0)$
subject to optional system dynamics
$s_{t+1}=f(s_t,u_t), \quad u_t \in \mathcal{U},$
and state or control constraints.
The gradients of $\sRob(\xi,\phi,0)$ can be computed using automatic differentiation: $\nabla_{\xi} \sRob(\xi,\phi,0) = \left[ \frac{\partial \sRob}{\partial s_0}, \frac{\partial \sRob}{\partial s_1}, \ldots, \frac{\partial \sRob}{\partial s_T} \right].$ The trajectory can then be iteratively updated using gradient-based optimization,
$\xi^{k+1} = \xi^{k} + \alpha_k \nabla_{\xi}\sRob(\xi^{k},\phi,0),$
where $\alpha_k$ is the step size. This formulation transforms logical satisfaction into a continuous optimization problem, where robustness gradients provide dense feedback that progressively resolves spatial violations and steers trajectories toward specification satisfaction without combinatorial search, as shown in Fig.~\ref{fig:exp_opt_qual}.

% Given a predefined spatio-temporal specification $\phi$, the trajectory $\xi$ becomes the optimizable variable. The smooth formula robustness $\sRob(\xi, \phi, 0)$ acts as a dense, continuous reward signal. Using standard auto-differentiation and gradient descent algorithms (e.g., Adam), we can iteratively update $\xi$ to maximize $\sRob$. As demonstrated qualitatively in Fig.~\ref{fig:exp_opt_qual}, this gradient flow naturally resolves spatial violations and safely steers the system toward constraint satisfaction without relying on combinatorial search.

\subsubsection{Specification Inference}
Conversely, given a dataset of expert demonstrations $\mathcal{D}=\{\xi_i\}_{i=1}^{N}$, the framework also enables \emph{learning the parameters of the specification}.  
Let the spatial-temporal specification be parameterized by
$\theta \in \Theta \subset \mathbb{R}^p, \ \phi = \phi_\theta,$
where $\theta$ represents continuous geometric quantities such as safety margins, object spatial relations. A spatial predicate can be written as
$\mu_\theta(x) = g(x;\theta) \ge 0.$
Using the differentiable robustness functional, each trajectory induces
$\rho_i(\theta) = \sRob(\xi_i,\phi_\theta,0),$
which measures the satisfaction margin of $\xi_i$. The parameters are learned by minimizing a robustness-based loss
$\min_{\theta \in \Theta}
\mathcal{L}(\theta)
=
\frac{1}{N}\sum_{i=1}^{N}
\ell\!\left(\rho_i(\theta)\right).$
Since spatial predicates and logical operators are implemented as smooth tensor operations, $\rho_i(\theta)$ is differentiable and gradients can be obtained via automatic differentiation:
$\nabla_\theta \mathcal{L}(\theta)
=
\frac{1}{N}\sum_{i=1}^{N}
\ell'\!\left(\rho_i(\theta)\right)
\nabla_\theta \sRob(\xi_i,\phi_\theta,0).$
This enables end-to-end inference of geometric specification parameters directly from demonstrations.

\section{EXPERIMENTS}
\label{sec:experiments}

We evaluate \emph{Differentiable SpaTiaL} on two capabilities enabled by our tensorized spatial semantics:
(i) gradient-based spatio-temporal trajectory optimization in cluttered scenes, and
(ii) learning continuous spatial specification parameters from demonstrations.
All experiments are implemented in PyTorch and executed on a single NVIDIA RTX 4090 GPU.
Unless stated otherwise, we use $\tau=10^{-2}$ for smooth $\min/\max$ (LogSumExp), boundary sampling density $S=16$ points per edge, and Minkowski support sampling with $\texttt{NUM\_DIR}=128$ directions.

\subsection{Task I: Spatio-Temporal Trajectory Optimization}
\label{subsec:exp_opt}

\paragraph{Specification}
We consider a planar manipulation proxy where the end-effector (or object footprint) must remain safely separated from obstacles while eventually reaching a goal region.
A representative specification is
\begin{equation*}
\phi_{\mathrm{opt}}
=
G\bigl(\FarE(\ent{EE}, \ent{Obs})\bigr)\ \wedge\
F_{[0,T]}\bigl(\EnclIn(\ent{EE}, \ent{Goal})\bigr),
\end{equation*}
where EE denotes the robot end-effector, Obs denotes the obstacles, and Goal is the target region.

\paragraph{Optimization objective}
We maximize the smoothed robustness by minimizing a hinge loss with a small trajectory smoothness regularizer:
\begin{equation}
\mathcal{L}(\xi) = \max(0,\eta-\sRob(\xi,\phi_{\mathrm{opt}},0)) + \lambda\mathcal{R}(\xi),
\end{equation}
where $\eta>0$ is the desired robustness margin and $\mathcal{R}(\xi)$ encourages smooth motion in pose space.

\paragraph{Qualitative results}
Fig.~\ref{fig:exp_opt_qual} shows typical optimization behavior.
Starting from an infeasible initialization (e.g., obstacle penetration and/or failure to achieve enclosure), gradients backpropagated through smooth SAT and boundary-sampled signed distance progressively push the trajectory toward feasibility.
In practice, we observe three consistent phases:
(i) a rapid ``collision resolution'' stage where repulsive gradients remove overlaps,
(ii) a ``constraint shaping'' stage where the path adjusts to satisfy \emph{Always}-type clearance, and
(iii) a ``task completion'' stage where the \emph{Eventually} clause becomes positive and the robustness margin increases.

\paragraph{Convergence}
The rightmost panel of Fig.~\ref{fig:exp_opt_qual} plots robustness (and/or loss) over iterations.
We typically observe a monotonic increase in robustness after a short warm-up period, indicating that the spatial gradients remain informative in both separated and intersecting regimes.
Compared to optimization driven by sparse collision signals, our smooth predicates yield dense gradients that reduce stagnation near contact configurations.
This is particularly important in our setting because the optimization must pass through both separated and near-contact configurations.
A useful surrogate should therefore remain informative not only when collisions occur, but also before contact, when the trajectory still needs to be steered toward a feasible region.

We also note a practical aspect of the optimization.
When all trajectory points evolve under a locally symmetric gradient field, the updates across time steps can become nearly identical, which may cause optimization to stall.
In such cases, adding small perturbations helps break the symmetry and improve progress.
We further observe that satisfying the formula does not always yield the most desirable trajectory.
Depending on the task, one may incorporate additional objectives such as smoothness regularization or rewards for earlier completion, so that optimization favors not only formula satisfaction but also overall trajectory quality.
\subsection{Task II: Learning Spatial Specification Parameters}
\label{subsec:exp_learning}

We show that spatial specifications can be automatically extracted from demonstrations by leveraging the differentiability of our robustness measure.

Given $30$ demonstration trajectories from a 3D pick-and-place task with three box obstacles in a $[0,10]^3$ workspace
(Fig.~\ref{fig:exp_learning_vis}), where the robot base is at
$(5,7,0)$, we proceed in two steps.
\textbf{(1)~Discovery.}
For each task phase and obstacle, we enumerate candidate spatial predicates (e.g., $\RightOf$, $\Above$) combined with temporal operators ($G_{I}$, $F_{I}$) bounded to the phase interval $I$.
A candidate $\varphi_k$ is retained only if it is satisfied by every demonstration, i.e.\ its worst-case robustness across $\mathcal{D}$ is positive.
We keep the two most robust candidates per phase--obstacle pair and conjoin them into a compound specification $\Phi = \bigwedge_k \varphi_k$.
\textbf{(2)~Margin optimization.}
Each selected predicate $\varphi_k$ is augmented with a learnable margin $\varepsilon_k \geq 0$.
We maximise the total margin while ensuring all demonstrations remain satisfied:
\begin{equation}
    \max_{\{\varepsilon_k\}} \; \sum_k \varepsilon_k
    \ \text{s.t.} \ \min_{\xi \in \mathcal{D}} \left[ \min_k \bigl( \sRob(\xi, \varphi_k, 0) - \varepsilon_k \bigr) \right] \geq 0.
\end{equation}
The learned $\varepsilon_k$ values represent the tightest safety margins that all demonstrations respect.
Table~\ref{tab:discovered_spec} summarizes the discovered specification that characterizes the spatial orientation of the arm trajectories relative to each obstacle. The learned margins further indicate how much slack each relation admits while remaining consistent with all demonstrations.
\begin{figure*}[t]
    \centering
    \includegraphics[width=0.3\textwidth]{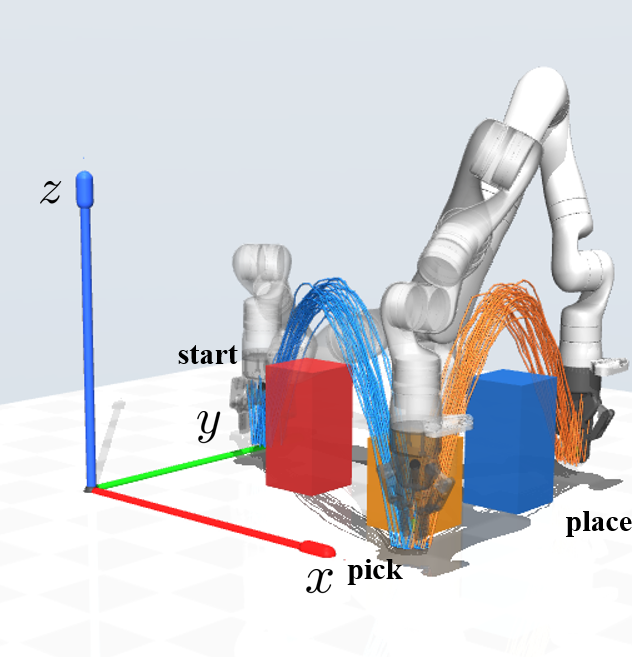}
    \hfill
    \includegraphics[width=0.3\textwidth]{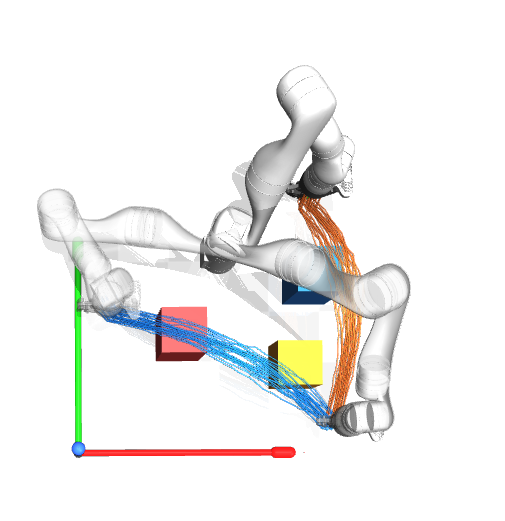}
    \hfill
    \includegraphics[width=0.3\textwidth]{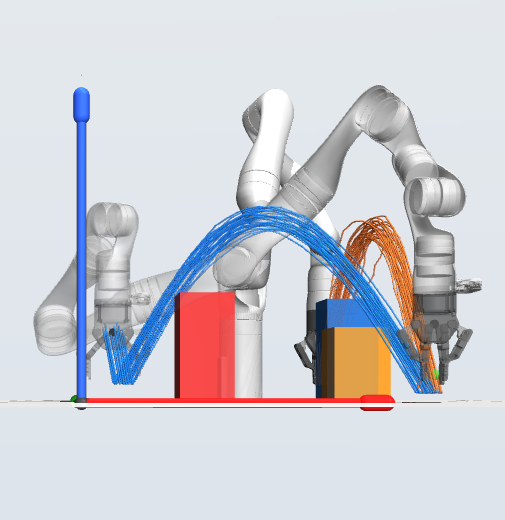}
    \caption{Learning spatial specification parameters from demonstrations via robustness backpropagation. From left to right: oblique, xy, and xz views.}
    \label{fig:exp_learning_vis}
    % \vspace{-1em}
\end{figure*}
\begin{table}[t]
\centering
\caption{Discovered spatio-temporal specification with learned margins; phase~1 ($I_1 = [0,59]$) and phase~2 ($I_2 = [60,118]$). All predicates are of the form $(\ent{arm}, O_k)$.}
\label{tab:discovered_spec}
\setlength{\tabcolsep}{2pt}
\footnotesize
\begin{tabular}{c l r @{\hskip 6pt} c l r}
\toprule
\multicolumn{3}{c}{\textbf{Start$\to$Pick}} & \multicolumn{3}{c}{\textbf{Pick$\to$Place}} \\
\cmidrule(lr){1-3} \cmidrule(lr){4-6}
\textbf{Obs} & \textbf{Formula} & $\boldsymbol{\varepsilon}$ & \textbf{Obs} & \textbf{Formula} & $\boldsymbol{\varepsilon}$ \\
\midrule
\multirow{2}{*}{\textcolor{red}{$O_1$}}
  & $F_{I_1}(\RightOf)$ & $4.25$
  & \multirow{2}{*}{\textcolor{red}{$O_1$}}
  & $G_{I_2}(\RightOf)$ & $3.31$ \\
  & $F_{I_1}(\Behind)$ & $1.87$
  & & $F_{I_2}(\InFront)$ & $3.77$ \\
\cmidrule(lr){1-3} \cmidrule(lr){4-6}
\multirow{2}{*}{\textcolor{orange}{$O_2$}}
  & $F_{I_1}(\LeftOf)$ & $5.25$
  & \multirow{2}{*}{\textcolor{orange}{$O_2$}}
  & $F_{I_2}(\InFront)$ & $4.73$ \\
  & $F_{I_1}(\Above)$ & $2.55$
  & & $F_{I_2}(\Above)$ & $2.52$ \\
\cmidrule(lr){1-3} \cmidrule(lr){4-6}
\multirow{2}{*}{\textcolor{blue}{$O_3$}}
  & $F_{I_1}(\LeftOf)$ & $5.83$
  & \multirow{2}{*}{\textcolor{blue}{$O_3$}}
  & $F_{I_2}(\Behind)$ & $3.66$ \\
  & $F_{I_1}(\Behind)$ & $3.76$
  & & $F_{I_2}(\Above)$ & $1.57$ \\
\bottomrule
\end{tabular}
\end{table}

\subsection{Task III: Geometric Accuracy Analysis}
\label{sec:geometric}
A natural question for smooth geometric relaxations is whether they remain numerically consistent with classical geometry engines.
To assess this, we compare our differentiable predicates against Shapely on randomized convex polygon pairs under random translations and rotations.
We consider four quantities: unsigned distance, signed distance, penetration depth, and enclosure robustness.
For each quantity, Fig.~\ref{fig:accuracy_shapely} plots the differentiable output against the Shapely value.
% Overall, the agreement is strong across all four metrics.
% The samples are tightly concentrated around the reference line \(y=x\), with \(R^2\) values of 0.9875, 0.9827, 0.9560, and 0.9835 for unsigned distance, signed distance, penetration depth, and enclosure robustness, respectively.
% The median relative errors are 7.1\%, 6.4\%, 6.6\%, and 0.9\%.
These results indicate that the proposed smooth predicates preserve the underlying geometric quantities with good fidelity, while remaining fully differentiable.

As shown in Fig.~\ref{fig:tau_sensitivity}, decreasing $\tau$ makes the approximation closer to the ground-truth geometric quantity, while also producing sharper gradients near transition regions. Increasing the boundary sampling density has a smaller effect, but generally improves numerical fidelity. We therefore choose these parameters to balance geometric accuracy and optimization stability in the experiments.

\begin{figure}[t]
\centering
    \includegraphics[width=\columnwidth]{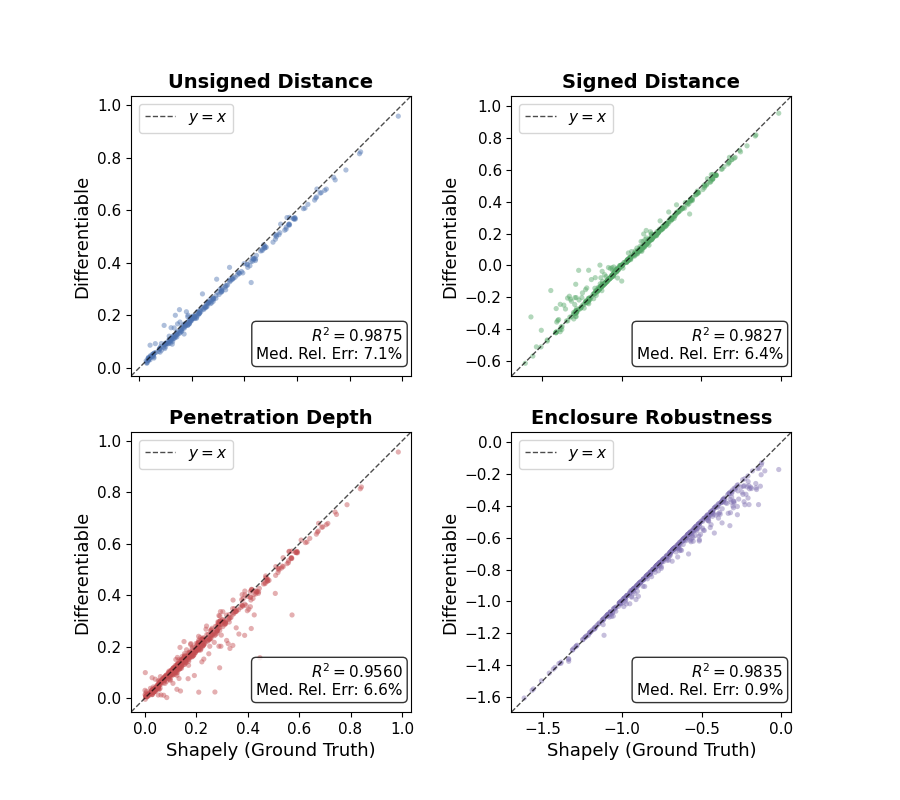} 
\caption{Geometric accuracy of differentiable predicates. We compare distance, signed distance, penetration, and enclosure robustness against Shapely across randomized convex polygons and transforms.}
\label{fig:accuracy_shapely}
% \vspace{-2em}
\end{figure}

\begin{figure}[t]
    \centering
    \includegraphics[width=\columnwidth]{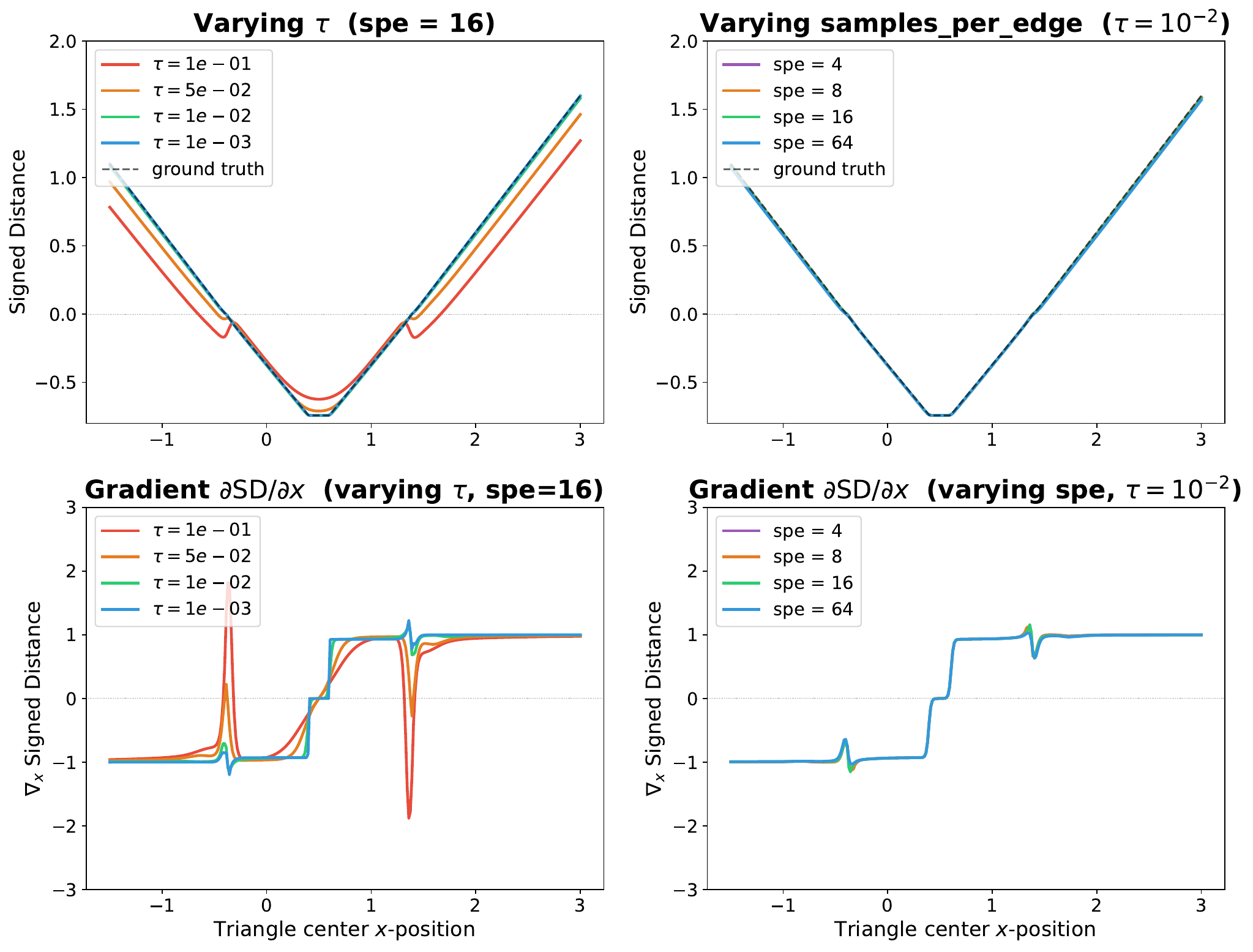}
    \caption{Effect of the smoothing temperature $\tau$ and boundary sampling density on signed distance and its gradient. }
    \label{fig:tau_sensitivity}
    \vspace{-0.7em}
\end{figure}

% \section{Conclusion}
% In this paper, we presented \emph{Differentiable SpaTiaL}, a novel framework that bridges formal SpaTiaL and gradient-based robotic optimization. By formulating smooth, tensorized approximations of geometric predicates such as SAT-based penetration and boundary-sampled distance, we established an unbroken gradient flow from high-level semantics to physical states. Our results demonstrate that this architecture enables efficient trajectory synthesis and specification learning in complex, cluttered environments. Future work will extend these differentiable semantics to 3D mesh representations and investigate their integration with deep reinforcement learning policies.
\section{Conclusion}
In this paper, we presented \emph{Differentiable SpaTiaL}, a framework that connects formal SpaTiaL with gradient-based robotic optimization.
By introducing smooth, tensorized approximations of geometric predicates such as SAT-based penetration and boundary-sampled distance, we maintain gradient flow from high-level semantic specifications to continuous physical states.
This makes it possible to use SpaTiaL not only for verification, but also for trajectory optimization and parameter learning.

Our experiments show that the proposed formulation supports efficient trajectory synthesis and specification learning in cluttered environments, while remaining numerically consistent with a classical geometry engine.
More broadly, the framework provides a practical way to combine symbolic spatio-temporal reasoning with modern autograd-based pipelines.

Future work will extend these differentiable semantics to 3D meshes and more complex object representations.
Another natural direction is to integrate them with reinforcement learning or imitation learning, where spatial robustness can provide structured guidance for long-horizon manipulation tasks.
\bibliographystyle{IEEEtran}
\bibliography{ref}

\end{document}